\documentclass[letterpaper]{article}
\usepackage{aaai2027}
\nocopyright
\usepackage[hyphens]{url}
\usepackage{graphicx}
\usepackage{natbib}
\usepackage{caption}
\usepackage{amsmath}
\usepackage{amssymb}
\usepackage{amsthm}

\usepackage{algorithm}
\usepackage{algorithmic}

\usepackage{booktabs}

\theoremstyle{plain}
\newtheorem{lemma}{Lemma}
\newtheorem{proposition}[lemma]{Proposition}
\newtheorem*{lemmavar}{Lemma 2$'$}
\newtheorem{theorem}{Theorem}
\theoremstyle{remark}

\newcommand{\Pnull}{P_{0}}
\newcommand{\Fnorm}[1]{\lVert #1 \rVert_{F}}
\newcommand{\oons}{\textsc{oons}}
\DeclareMathOperator{\col}{col}

\title{FedAlphaEdit: Null-Space-Aligned Merging for Collaborative Knowledge Editing}

\author{
    Sota Sugawara\textsuperscript{\rm 1},
    Yukihiko Okada\textsuperscript{\rm 2,\rm 3}
}
\affiliations{
    \textsuperscript{\rm 1}Graduate School of Science and Technology, University of Tsukuba, Tsukuba, Japan\\
    \textsuperscript{\rm 2}Center for Artificial Intelligence Research, Tsukuba Institute for Advanced Research, University of Tsukuba, Tsukuba, Japan\\
    \textsuperscript{\rm 3}Institute of Systems and Information Engineering, University of Tsukuba, Tsukuba, Japan\\
    s2620520@u.tsukuba.ac.jp
}

\begin{document}

\maketitle

\begin{abstract}
Multiple institutions may each hold their own private knowledge edits and wish to integrate them into a single large language model without sharing raw edit requests. Null-space-constrained editing methods such as AlphaEdit mathematically guarantee that each update leaves unrelated knowledge intact, while collaborative frameworks such as CollabEdit aggregate edits from multiple clients without data sharing. Combining the two appears trivial. However, we show that this naive combination fails structurally, and we identify its cause. Guided by this analysis, we propose FedAlphaEdit. To our knowledge, this is the first collaborative knowledge editing framework that aligns both local editing and the server-side merging rule under a single null-space principle for preserving existing knowledge. FedAlphaEdit builds on null-space-aligned merging, in which clients share projected statistics and the server provably recovers the result of editing everything in one place under a one-shot idealization. Empirically, the proposed method repairs the collapse and brings edit success and preservation simultaneously close to the level of centralized editing across two architecture families. FedAlphaEdit thus lets institutions that cannot share raw edit data, such as hospitals and financial firms, jointly maintain a shared model that closely approximates editing all facts in one place.
\end{abstract}

\section{Introduction}
\label{sec:intro}

Large language models encode factual knowledge that becomes stale or wrong over time. Since retraining the entire model just to correct a few facts is not cost-effective, \emph{knowledge editing} revises only the relevant facts directly. Locate-then-edit methods such as ROME and MEMIT compute closed-form updates to feed-forward weights for this purpose \citep{meng2022rome,meng2023memit}. Applying many edits in sequence, however, degrades unrelated knowledge and general abilities \citep{yao2023editing}. AlphaEdit addresses this by projecting each update onto the null space of the keys of preserved knowledge, so that, for an exact null space, preserved associations are left unchanged by construction \citep{fang2025alphaedit}. This preservation guarantee holds within a single editing party, and what happens when the edits of several \emph{clients} are merged has not been examined.

In practice, institutions such as hospitals correcting clinical knowledge and financial firms applying regulatory updates are routinely barred from sharing raw edit requests, yet each needs its own edits reflected in a shared model, so the edits of several institutions must be consolidated without exchanging raw edit requests \citep{zheng2025collabedit,mcmahan2017fedavg}. Keeping edit requests local is a data-locality property and is not by itself a privacy guarantee, a distinction we return to in Section~\ref{sec:limitations}. CollabEdit formalizes this collaborative setting and introduces a covariance-weighted merging rule that closely mimics applying all edits centrally, while communicating only key statistics \citep{zheng2025collabedit}.

Since AlphaEdit is plug-and-play, and CollabEdit both derives its merge without reference to the internal structure of the editor and is presented as general enough to integrate other KE methods \citep{zheng2025collabedit}, combining them looks natural. We show that this naive combination fails structurally. The difference from AlphaEdit itself is the following. AlphaEdit constrains a single editor's update locally, whereas our contribution is to align the server-side merging rule with the same null-space principle, and Lemmas~\ref{lem:merge} and~\ref{lem:cancel} prove that the local constraint alone does not survive a covariance-weighted merge.

Under an identical evaluation pipeline, the naive combination degrades along two axes at once (Table~\ref{tab:g1}). First, preservation breaks, as accuracy on knowledge neighboring the edited facts (neighborhood specificity) drops by $4.2$ points, even though each client's update satisfies the null-space constraint almost perfectly, while merging pushes up to $14\%$ of the merged update outside the null space. Second, and unexpectedly, editing itself breaks, as edit success falls by $17.4$ points. Both degradations grow monotonically with the number of clients $N$ at a fixed total edit budget, and vanish by definition at $N=1$ (Figure~\ref{fig:scaling}). This is a degradation specific to the collaborative scale, absent from any centralized configuration.

\begin{table}[b]
\centering
\begin{tabular}{lccc}
\toprule
 & Centralized & Naive & $\Delta$ \\
\midrule
ES    & $97.08$ & $79.69 \pm 0.88$ & $-17.39$ \\
NS    & $64.43$ & $60.24 \pm 0.11$ & $-4.19$ \\
\oons & ${\approx}3{\times}10^{-4}\%$ (floor) & $14.0\text{--}14.1\%$ & --- \\
\bottomrule
\end{tabular}
\caption{Naive combination versus the centralized reference on GPT-2 XL (MCF, 5000 edits, $10 \times 500$). Centralized is a single deterministic run; Naive is three-seed mean$\pm$SD. Metric definitions are in Sections~\ref{sec:exp-setup} and~\ref{sec:measurement}.}
\label{tab:g1}
\end{table}

\begin{figure*}[t]
\centering
\includegraphics[width=0.95\textwidth]{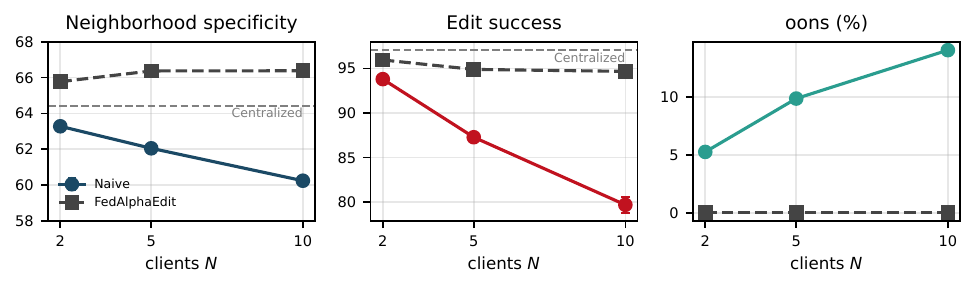}
\caption{The naive combination and FedAlphaEdit as the same total workload ($5000$ edits) is partitioned across a growing number of clients $N$: neighborhood specificity (left), edit success (middle), and \oons\ (right; the FedAlphaEdit series is at the measurement floor throughout). Dashed lines mark the centralized reference ($64.43$ and $97.08$). All points are three-seed mean$\pm$SD. The per-client edit count co-varies with $N$ (Section~\ref{sec:exp-naive}).}
\label{fig:scaling}
\end{figure*}

These two failures share a single cause. CollabEdit's merge weights exactly cancel the MEMIT-type normalizer (Lemma~\ref{lem:cancel}(a)), but AlphaEdit's normalizer has a different structure, so the cancellation no longer holds and both axes collapse.
Guided by this analysis, we derive, in the collaborative setting, an AlphaEdit analogue of the cancellation identity. If each client shares a projected covariance statistic and the server merges with the matching weights, the aggregate recovers global AlphaEdit editing \emph{exactly} (Theorem~\ref{thm:exact}, equations~\eqref{eq:nall}--\eqref{eq:proposed}), in the one-shot regime of a single solve per client; under genuine sequential editing the identity holds only approximately, and we quantify the resulting drift empirically. The communicated object is a covariance statistic of the same class as CollabEdit's, so no raw keys are shared. We call this \emph{null-space-aligned merging}. Empirically, it restores neighborhood specificity to the centralized level and \oons\ to the measurement floor, and closes most of the edit-success gap on both GPT-2 XL and Qwen2.5-1.5B.
\noindent In summary, we contribute:
\begin{itemize}
\item \textbf{A phenomenon.} We identify and quantify a \emph{double collapse}, of both editing and preservation, when null-space-constrained edits are merged by a covariance-weighted rule. The collapse is specific to the federated setting, growing monotonically with the number of clients, whether the total edit budget or the per-client edit count is held fixed, and vanishing at a single client.
\item \textbf{A mechanism.} An algebraic account of why the naive combination must break, independent of architecture.
\item \textbf{A method.} FedAlphaEdit, based on null-space-aligned merging, which aggregates projected statistics and comes with a one-shot exactness theorem, and contains CollabEdit's merge as its $\Pnull = I$ instance (Proposition~\ref{prop:editormatched}).
\end{itemize}

\section{Related Work}
\label{sec:related}

\subsection{Locate-then-Edit Knowledge Editing and Sequential Degradation}
ROME and MEMIT established locate-then-edit knowledge editing, computing closed-form, batched updates to feed-forward weights \citep{meng2022rome,meng2023memit}. A recurring finding is that applying edits in sequence gradually degrades unrelated knowledge and general language ability \citep{yao2023editing}. AlphaEdit mitigates this by constraining each update to the null space of the preserved-knowledge keys, giving a preservation guarantee under an exact null space \citep{fang2025alphaedit}. Because the null space is estimated from a finite key sample, the projector is only approximate; BetaEdit quantifies the resulting leakage and its accumulation under sequential editing in the single-user setting \citep{liu2026betaedit}. EvoEdit and LOKI carry null-space constraints further along this sequential axis, via an evolving null space and memory-free lifelong editing respectively \citep{lyu2026evoedit,eskandar2026loki}. All of this work operates along the \emph{time} axis, successive edits to a single model, and treats an approximate, inherited leak.

\subsection{Collaborative and Federated Knowledge Editing}
CollabEdit formalizes collaborative knowledge editing and proposes a non-destructive, covariance-weighted merge that mimics centralized global editing while sharing only key statistics rather than raw edit requests \citep{zheng2025collabedit,mcmahan2017fedavg}. Its non-destructiveness, however, rests on an exact cancellation with the MEMIT-type normalizer (Lemma~\ref{lem:cancel}, by our analysis); this cancellation is editor-dependent, and CollabEdit's behavior when combined with a null-space-constrained editor has not been examined. We show in Proposition~\ref{prop:editormatched} that CollabEdit's merge is the $\Pnull = I$ instance of the merge rule we derive. Federated locate-then-edit work such as FedLEKE targets communication and computation efficiency rather than preservation \citep{zhao2025fleke}. Its retrieval-based federation is orthogonal to the covariance-weighted merge structure we analyze. We include a reimplementation as a reference row (Table~\ref{tab:g2res}). To our knowledge, FedAlphaEdit is the first collaborative knowledge editing framework that aligns both local editing and the server-side merging rule under a single null-space preservation principle. Task-vector merging methods such as Task Arithmetic and TIES-Merging study interference as a loss of task performance \citep{ilharco2023task,yadav2023ties}, and null-space projections have also been applied at the merge stage for task-vector and LoRA fusion \citep{qiu2026nufilt,qiu2025mingle}. Neither line addresses our setting, in which closed-form updates derived from key covariances carry a preservation guarantee that must survive aggregation.

\section{Method}
\label{sec:method}

We first diagnose the failure (Sections~\ref{sec:lemma1}--\ref{sec:lemma4}), then derive the fix, null-space-aligned merging, an exact merge rule stated as Theorem~\ref{thm:exact} and implemented in Algorithm~\ref{alg:nsalign} (Section~\ref{sec:theorem}).

\subsection{Preliminaries and Notation}
\label{sec:prelim}
Let $W \in \mathbb{R}^{d \times d_0}$ be the edited feed-forward (down-projection) weight, and let $K_0 \in \mathbb{R}^{d_0 \times n_0}$ and $V_0$ collect the keys and values of preserved knowledge, with $W K_0 = V_0$ before editing. The AlphaEdit projector $\Pnull \in \mathbb{R}^{d_0 \times d_0}$ is constructed from the eigendecomposition of $K_0 K_0^{\top}$. Collecting the eigenvectors $\hat{U}$ whose eigenvalues fall below a threshold (specified in Section~\ref{sec:exp-setup}) gives $\Pnull = \hat{U}\hat{U}^{\top}$, which is symmetric and idempotent ($\Pnull = \Pnull^{\top}$, $\Pnull^2 = \Pnull$). In the idealized case $\Pnull K_0 = 0$ exactly, while in the realistic (approximate) case $\Fnorm{\Pnull K_0} \le \varepsilon$. Preservation means $(W+\Delta)K_0 = V_0$, equivalently $\Delta K_0 = 0$.

For clients $i = 1,\dots,N$, let $K_i$ and $V_i$ be the edit keys and target values, and write $A_i := K_{p,i}K_{p,i}^{\top} + K_i K_i^{\top}$ for the symmetric positive semidefinite second-moment matrix ($K_{p,i}$ accumulates a client's own prior edits). Each client's AlphaEdit update is $\Delta_i = R_i K_i^{\top}\,\Pnull (A_i \Pnull + \beta I)^{-1}$, with residual $R_i = V_i - W K_i$ and ridge $\beta$. Each client sends its local update $\Delta_i$ to the server, which consolidates them into a single update $\Delta_{\text{naive}}$ to the shared model. CollabEdit performs this consolidation with a covariance-weighted merge, which we write in the abstract form
\begin{equation}
\Delta_{\text{naive}} = \sum_i \Delta_i\, M_i, \qquad M_i = C_i\, C_{\text{tot}}^{-1},
\label{eq:merge-abstract}
\end{equation}
where $C_i = K_i K_i^{\top}$ and $C_{\text{tot}} = \sum_j C_j$, so $\sum_i M_i = I$ (a matrix-valued weighted average). The released CollabEdit implementation regularizes this merge with a shared base covariance $\lambda C$ ($\lambda$ a scalar update weight, $C$ a shared second-moment matrix; values in Appendix~\ref{app:impl}), so that the merge weights become $M_i = (C_i + \lambda C)(C_{\text{tot}} + \lambda C)^{-1}$, with $\lambda C$ added to both the per-client covariance $C_i$ and the total $C_{\text{tot}}$. Because $\lambda C$ is added to each client's weight but only once to the total, a surplus arises whose consequence for edit success is analyzed in Section~\ref{sec:lemma4}.

\subsection{A Right-Projected Form of the Update}
\label{sec:lemma1}
\begin{lemma}[Right-projected form]
\label{lem:rightproj}
The AlphaEdit update admits the factorization $\Delta = D\Pnull$ for some $D$.
\end{lemma}
The proof, an elementary consequence of the symmetry and idempotency of $\Pnull$, is in Appendix~\ref{app:proof-thm1}.
In particular, $\Delta K_0 = D\Pnull K_0 = 0$ whenever $\Pnull K_0 = 0$. Since only the symmetry and positive semidefiniteness of $A$ are used, the factorization survives cross-round accumulation of $K_{p,i}$.

\subsection{Merging Breaks Preservation}
\label{sec:lemma2}
\begin{lemma}[Merging breaks null-space preservation]
\label{lem:merge}
Let $\Delta_i = D_i \Pnull$ with $\Pnull K_0 = 0$, and let the server form the merged update of the naive combination, $\Delta_{\text{naive}} = \sum_i \Delta_i M_i$, with $M_i = C_i C_{\text{tot}}^{-1}$ (Section~\ref{sec:prelim}), $C_i = K_i K_i^{\top}$, and $C_{\text{tot}} = \sum_i C_i$ invertible. Then
\begin{equation}
\Delta_{\text{naive}} K_0 = \sum_i D_i\, [\Pnull,\, M_i]\, K_0,
\label{eq:commutator}
\end{equation}
where $[\Pnull, M_i] := \Pnull M_i - M_i \Pnull$. In particular, although every client update exactly preserves $K_0$, the merged update is no longer guaranteed to do so. Commutation of every merge weight $M_i$ with $\Pnull$ makes every term of \eqref{eq:commutator} vanish and is thus sufficient for preservation. When some $M_i$ fails to commute with $\Pnull$, preservation generically fails, because the surviving terms could cancel across clients only through a coincidence among the private $D_i$, which depend on each client's private edit content.
\end{lemma}
\begin{proof}
$\Delta_{\text{naive}} K_0 = \sum_i D_i \Pnull M_i K_0$, and for each term $\Pnull M_i K_0 = M_i \Pnull K_0 + [\Pnull, M_i] K_0 = [\Pnull, M_i] K_0$ since $\Pnull K_0 = 0$.
\end{proof}

The proof of \eqref{eq:commutator} makes no use of the specific form of the merge weights, so the identity holds verbatim for arbitrary $M_i$, including the regularized weights $(C_i + \lambda C)(C_{\text{tot}} + \lambda C)^{-1}$ of the released implementation (Appendix~\ref{app:impl}).

\begin{lemmavar}[Approximate case]
Under $\Fnorm{\Pnull K_0} \le \varepsilon$,
\begin{equation}
\Delta_{\text{naive}} K_0 = \underbrace{\sum_i D_i [\Pnull, M_i] K_0}_{\text{structural term}} + \underbrace{\sum_i D_i M_i (\Pnull K_0)}_{\text{inherited term}},
\end{equation}
where the inherited term is bounded by $\sum_i \lVert D_i \rVert_2 \lVert M_i \rVert_2 \cdot \varepsilon$, whereas each client's own leakage $\Fnorm{\Delta_i K_0} \le \lVert D_i \rVert_2\, \varepsilon$ involves no structural term. The distinctive evidence of a structural leak is \oons\ (Section~\ref{sec:measurement}).
\end{lemmavar}

The structural term exists only because there is a merge step, and is thus absent by definition from centralized and single-user editing. The inherited term is the single-user approximate-null-space leak, carried through the merge weights.

\begin{proposition}[When preservation survives]
\label{prop:degenerate}
Preservation holds, i.e. $\Delta_{\text{naive}} K_0 = 0$, if the merge weights are scalar ($M_i = \alpha_i I$), in particular for simple averaging $\alpha_i = 1/N$.
\end{proposition}
\begin{proof}
For scalar weights $[\Pnull, \alpha_i I] = 0$, so \eqref{eq:commutator} vanishes.
\end{proof}

The commuting condition also identifies the driver of the failure. If every client's edit-key covariance is free of cross-covariance between the preserved subspace $\col(K_0)$ and its orthogonal complement, every merge weight commutes with $\Pnull$ and preservation holds. In practice this sufficient condition is rarely available. Under an approximate null space the preserved-key second moment $K_0 K_0^{\top}$ is full-rank, so edit keys typically correlate with the preserved statistics and cross-covariance is generically present. The condition is then lost in general and preservation fails. Moreover, $C_{\text{tot}}^{-1}$ is shared, so a single client's cross-covariance corrupts every merge weight. Empirically, the merged update places up to $14\%$ of its norm share outside the null space (Section~\ref{sec:exp-naive}).

\subsection{The Cancellation Identity and the Edit-Success Collapse}
\label{sec:lemma4}
The loss of specificity above is only half the phenomenon; the collapse of edit success has the same algebraic root.
\begin{lemma}[Cancellation identity and its breaking]
\label{lem:cancel}
\emph{(a) MEMIT.} For MEMIT-type updates $\Delta_i = R_i K_i^{\top}(C_i + \lambda C)^{-1}$, the CollabEdit weights cancel exactly, $\Delta_i(C_i + \lambda C) = R_i K_i^{\top}$, so that $\Delta_{\text{naive}} = \sum_i R_i K_i^{\top}(C_{\text{tot}} + \lambda C)^{-1} = \Delta_{\text{central}}^{\mathrm{MEMIT}}$ exactly, under the same one-shot, common-state evaluation as Theorem~\ref{thm:exact}.
\emph{(b) AlphaEdit.} The root cause is that AlphaEdit's normalizer $\Pnull A_i \Pnull + \beta I$ has a different structure from $C_i + \lambda C$, so the CollabEdit weights do not cancel and an error $E := \Delta_{\text{naive}} - \Delta_{\text{central}}^{\mathrm{Alpha}}$ remains. It decomposes as
\begin{equation}
E(I-\Pnull) = \Delta_{\text{naive}}(I-\Pnull) \quad\text{and}\quad E\Pnull,
\label{eq:errsplit}
\end{equation}
the first being the out-of-null-space leak measured by \oons\ (loss of specificity) and the second a distortion of the edit within $\col(\Pnull)$ (the collapse of edit success).
\end{lemma}
Because this mismatch is structural rather than a matter of the value of $\lambda C$, re-tuning $\lambda C$ cannot restore the cancellation identity in general. Since $\lambda C$ enters each of the $N$ per-client terms but only once in the global normalizer, the summed normalizers exceed the latter by $(N-1)\lambda C$, amplifying the distortion as $N$ grows. At $N=1$ the weights reduce to the identity and the distortion is exactly zero. Varying the shared-covariance weight by a factor of two in either direction does not resolve the collapse in edit success or specificity (Table~\ref{tab:lamc} in Appendix~\ref{app:exp}).

Thus the observed edit-success drop and the null-space leak are not two independent bugs but the $\Pnull$- and $(I-\Pnull)$-components of a \emph{single} broken identity.

\subsection{Null-Space-Aligned Merging}
\label{sec:theorem}
\label{sec:measurement}
The commutativity of Lemma~\ref{lem:rightproj} supplies an AlphaEdit analogue of the cancellation identity, and hence a merge that is exact and that, by Proposition~\ref{prop:editormatched}, contains CollabEdit's merge as a special case.
\begin{theorem}[Exactness of null-space-aligned merging]
\label{thm:exact}
Suppose each client computes its AlphaEdit update $\Delta_i = R_i K_i^{\top}\Pnull(A_i \Pnull + \beta I)^{-1}$ with a single solve and shares the projected statistic $\Pnull A_i \Pnull$. Let
\begin{gather}
N_{\text{all}} := \Big(\sum_i \Pnull A_i \Pnull\Big) + \beta I, \label{eq:nall}\\
\Delta_{\text{proposed}} := \Big[\sum_i \Delta_i\,(\Pnull A_i \Pnull + \beta I)\Big] N_{\text{all}}^{-1}. \label{eq:proposed}
\end{gather}
Then $\Delta_{\text{proposed}} = \Delta_{\text{central}}^{\mathrm{Alpha}}$ exactly, where $\Delta_{\text{central}}^{\mathrm{Alpha}}$ is global AlphaEdit editing applied to the pooled edits. Moreover $\Delta_{\text{proposed}}(I-\Pnull) = 0$, so the merge introduces no out-of-null-space component. Preservation is exact whenever $\Pnull K_0 = 0$, and under an approximate null space ($\Fnorm{\Pnull K_0} \le \varepsilon$) the merged leakage coincides with that of centralized editing, containing the inherited term of Lemma~2$'$ only and no structural term.
\end{theorem}
\noindent\emph{Proof sketch.} By Lemma~\ref{lem:rightproj}, each per-client normalizer cancels against its own inverse once $\Pnull$ commutes through, and the same commutation applied to $N_{\text{all}}^{-1}$ assembles the pooled global solve. The full proof is in Appendix~\ref{app:proof-thm1}.

\begin{algorithm}[t]
\caption{Null-space-aligned merging, implementing the core merge rule \eqref{eq:proposed} (one-shot regime)}
\label{alg:nsalign}
\textbf{Pre-shared}: projector $\Pnull$, ridge $\beta$\\
\textbf{Received from each client} $i$: update $\Delta_i$, projected statistic $\Pnull A_i \Pnull$\\
\textbf{Output}: merged update $\Delta_{\text{proposed}}$
\begin{algorithmic}[1]
\STATE \textbf{each client} $i$ \textbf{in parallel}:
\STATE \quad compute $\Delta_i = R_i K_i^{\top}\Pnull(A_i \Pnull + \beta I)^{-1}$
\STATE \quad send $\Delta_i$ and the projected statistic $\Pnull A_i \Pnull$
\STATE \textbf{server}:
\STATE \quad $N_{\text{all}} \leftarrow \big(\sum_i \Pnull A_i \Pnull\big) + \beta I$
\STATE \quad $\Delta_{\text{proposed}} \leftarrow \big[\sum_i \Delta_i(\Pnull A_i \Pnull + \beta I)\big] N_{\text{all}}^{-1}$
\STATE \textbf{return} $\Delta_{\text{proposed}}$
\end{algorithmic}
\end{algorithm}

Algorithm~\ref{alg:nsalign} states the procedure of \eqref{eq:nall}--\eqref{eq:proposed}. The ridge $\beta$ must enter $N_{\text{all}}$ exactly once. Adding it per client would reproduce the same over-counting structure as the $(N-1)\lambda C$ surplus of Section~\ref{sec:lemma4}, with the ridge in place of the shared covariance. Appendix~\ref{app:exp} reports a unit test that detects exactly this misconfiguration.

\noindent\textbf{Scope of the exactness claim.} Theorem~\ref{thm:exact} holds under a one-shot idealization: one solve per client, a shared projector $\Pnull$, each client's accumulated prior-edit second moment $K_{p,i}K_{p,i}^{\top}$ summing to the global prior-edit second moment (in the single-round setting both are zero), and, crucially, all key statistics ($K_i$ and $A_i$) evaluated at a common model state (single-layer editing, or when inter-layer state divergence is negligible). Under multi-layer sequential editing, per-client edits to lower layers shift the hidden states from which higher-layer keys are computed, so the per-client $K_i$ and $A_i$ diverge from their global counterparts and exactness breaks; we quantify the resulting drift empirically, including a layer-wise operator-divergence diagnostic. No assumption on how the edits are distributed across clients enters the premises of Theorem~\ref{thm:exact}, so under the one-shot, common-state idealization the exact equivalence holds for any partition, including non-IID and imbalanced ones. The partition can affect the result only through violations of these premises, namely the inter-layer state divergence of multi-layer sequential editing and the batch-dependent key statistics of Section~\ref{sec:exp-diag}.

\begin{proposition}[Editor-matched merging and the $\Pnull = I$ instance]
\label{prop:editormatched}
Theorem~\ref{thm:exact} and Lemma~\ref{lem:rightproj} hold verbatim when the ridge $\beta I$ is replaced by any symmetric positive definite $\Lambda$ with $\Pnull\Lambda = \Lambda\Pnull$, since the proofs use only this commutation (Appendix~\ref{app:proof-thm1}). Take $\Pnull = I$ (no null-space constraint), $\Lambda = \lambda C$, and the single-round second moment $A_i = C_i$. Then the per-client update becomes the MEMIT update $R_i K_i^{\top}(C_i + \lambda C)^{-1}$, the normalizers become $N_i = C_i + \lambda C$ and $N_{\text{all}} = C_{\text{tot}} + \lambda C$, and \eqref{eq:proposed} reduces to $\sum_i \Delta_i (C_i + \lambda C)(C_{\text{tot}} + \lambda C)^{-1}$, which is exactly the merge of the released CollabEdit implementation (Appendix~\ref{app:impl}); the conclusion of Theorem~\ref{thm:exact} reduces to Lemma~\ref{lem:cancel}(a). CollabEdit is therefore the $\Pnull = I$ instance of the same editor-matched merge rule, and the naive combination fails precisely because it keeps the $\Pnull = I$ weights while the editor's normalizer has $\Pnull \neq I$.
\end{proposition}

\noindent\textbf{Communication and privacy.} Beyond $\Delta_i$, clients share only $\Pnull A_i \Pnull$, a $d_0 \times d_0$ statistic of the same class as CollabEdit's key-covariance sharing; no raw keys are transmitted. In fp32 this statistic occupies $163.8$\,MB per edited layer for GPT-2 XL, so each client sends $819.2$\,MB per round over the five edited layers, and about $1.6$\,GB for Qwen2.5-1.5B (Appendix~\ref{app:exp}).

\noindent\textbf{From theory to measurement.} We close by fixing the correspondence between the algebra above and the reported metrics. The out-of-null-space fraction (\oons), the scale-normalized share of the update outside $\col(\Pnull)$ averaged over the five edited layers, is $\Fnorm{\Delta(I-\Pnull)}/\Fnorm{\Delta}$ per layer and corresponds to the $E(I-\Pnull)$ (structural) component of Lemma~\ref{lem:cancel}(b) and hence to loss of specificity, while edit-success loss corresponds to the $E\Pnull$ component. Being a ratio of Frobenius norms, \oons\ is a norm share and not a ratio of squared norms, and it is reported in percent throughout. We report \oons\ rather than an unnormalized leakage norm because the latter is dominated by the overall update magnitude rather than by structural leakage.

\section{Experiments}
\label{sec:exp}
The main results (Sections~\ref{sec:exp-naive}--\ref{sec:exp-zsre}) are measured in the single-round multi-layer regime, where the common-state premise of Theorem~\ref{thm:exact} does not hold exactly; Section~\ref{sec:exp-diag} isolates the consequence.

\subsection{Experimental Setup}
\label{sec:exp-setup}
\noindent\textbf{Internal-contrast principle.} Two baselines anchor every table: centralized AlphaEdit (the centralized reference), and the naive combination of per-client AlphaEdit with CollabEdit merging.

\noindent\textbf{Models and hyperparameters.} We evaluate two architecture families. GPT-2 XL (1.5B) is our primary model, editing the \texttt{c\_proj} weights of layers 13--17, with second-moment update weight $\lambda = 20000$ and AlphaEdit ridge $\beta = 10$. Qwen2.5-1.5B Base is a second family for assessing generality, editing \texttt{down\_proj} (input dimension $8960$) of layers 4--8, with $\lambda = 15000$ and $\beta = 10$ (remaining hyperparameters in Appendix~\ref{app:exp}). Both use null-space threshold $0.02$, following the released AlphaEdit implementation. Two symbols recur below. $\beta$ is the AlphaEdit ridge added to the normalizer in Section~\ref{sec:prelim}; the released code names the same parameter \texttt{L2}, and the two are one parameter. $C_1$ is CollabEdit's dynamic covariance term, an optional statistic accumulated across rounds; it is fixed to $0$ throughout this work (Appendix~\ref{app:exp}).

\noindent\textbf{Data and configuration.} Our primary dataset is the MultiCounterFact (MCF) benchmark: $5000$ edits ($10$ clients $\times\,500$) on both GPT-2 XL and Qwen, in the single-round configuration with $C_1 = 0$. To test generality across datasets we additionally use zsRE, reported separately. Seeds are the client case-assignment randomness (IID sharding); the naive combination, FedAlphaEdit, and the four ablation arms use three seeds (reported as mean\,$\pm$\,SD).

\noindent\textbf{Metrics.} We report edit success (ES), the fraction of edits for which the model produces the target answer, and neighborhood specificity (NS), the fraction of neighborhood prompts on which the model retains its pre-edit answer, from the standard MCF evaluation, together with GLUE over six tasks and \oons.

\noindent\textbf{The \oons\ floor.} Storing $\Pnull$ in fp32 limits its idempotency to ${\approx}3{\times}10^{-4}\%$ (that is, $3\times10^{-6}$ as a fraction), which we take as the \oons\ measurement floor, so a floor-level \oons\ means the constraint is met to within measurement precision rather than zero leakage.

\noindent\textbf{Reproducibility and environment.} Every experiment runs on a single NVIDIA GeForce RTX 3070 Ti (8\,GB) under a $7.4$\,GB hard VRAM cap. Full experimental details are in Appendix~\ref{app:exp}. Code for the proposed method, the centralized and naive baselines, the merge ablations, and the Theorem~\ref{thm:exact} unit test is available at \url{https://github.com/soutasuga/FedAlphaEdit}; the remaining components are available from the authors upon request.

\subsection{The Naive Combination Breaks Both Axes}
\label{sec:exp-naive}
Sections~\ref{sec:lemma2} and~\ref{sec:lemma4} predict that covariance-weighted merging of null-space-constrained edits degrades both preservation and editing, and Table~\ref{tab:g1} confirms both failures on GPT-2 XL, where the naive combination loses $17.39$ points of edit success and $4.19$ of neighborhood specificity relative to the centralized reference.

\noindent\textbf{Structural leakage.} Each client's update stays at the \oons\ floor, yet the merged update places $14.0$--$14.1\%$ of its norm share outside the null space, the empirical signature of the commutator term of Lemma~\ref{lem:merge}.

\noindent\textbf{Scaling.} The degradation amplifies monotonically with the client count $N$. Figure~\ref{fig:scaling} varies $N$ at a fixed total budget: NS drops by $1.15 \to 2.38 \to 4.19$ points relative to the centralized reference at $N=2,5,10$, edit success falls $93.80 \to 87.28 \to 79.69$, and \oons\ rises $5.25 \to 9.86 \to 14.04\%$; at $N=1$ the merge weights are the identity and the degradation vanishes, matching the $N=1$ boundary of Lemma~\ref{lem:cancel}. Along the round axis, where the total edit count co-varies with $R$ ($1000/3000/5000$), the degradation shows an increasing trend, pronounced at five rounds (NS drop $0.32/0.35/1.69$). \emph{Confound.} In this fixed-budget series the per-client edit count co-varies with $N$. A control series holding the per-client count at $500$ while $N$ grows shows the same monotone degradation on all three metrics (ES $99.83 \to 95.03 \to 79.69$, three-seed mean, $N{=}10$ shared between the two series, Table~\ref{tab:npc} in Appendix~\ref{app:exp}). The total edit count co-varies with $N$ in this control, in the opposite direction to the fixed-budget confound. Under either fixed condition the degradation persists with $N$. A centralized baseline editing the same case sets declines only from $99.80$ to $99.24$ to $97.08$ across the same totals (Table~\ref{tab:npc}). The naive gap at matched total edit counts is indistinguishable at $N=2$ (within seed variation) and grows to $4.21$ and $17.39$ points at $N=5,10$, beyond what sequential degradation alone accounts for.

\subsection{Null-Space-Aligned Merging Recovers Preservation, and Most of Editing}
\label{sec:exp-recover}

\begin{table*}[t]
\centering
\begin{tabular}{lccc}
\toprule
 & ES & NS & \oons \\
\midrule
Centralized AlphaEdit & $97.08$ & $64.43$ & ${\approx}3{\times}10^{-4}\%$ (floor) \\
Centralized MEMIT     & $90.80$ & $64.42$ & $26.6\%$ \\
CollabEdit (MEMIT)    & $90.75 \pm 0.03$ & $64.34 \pm 0.13$ & $26.5 \pm 0.05\%$ \\
FedLEKE (our reimplementation) & $87.74 \pm 0.51$ & $66.78 \pm 0.54$ & $23.98 \pm 0.02\%$ \\
\midrule
Naive (AlphaEdit + CollabEdit)             & $79.69 \pm 0.88$ & $60.24 \pm 0.11$ & $14.0\text{--}14.1\%$ \\
CollabEdit (MEMIT) + null-space projection & \ensuremath{82.45 \pm 0.06} & \ensuremath{71.90 \pm 0.03} & ${\approx}3{\times}10^{-4}\%$ (floor) \\
Simple average (AlphaEdit clients)         & \ensuremath{31.23 \pm 0.22} & \ensuremath{77.16 \pm 0.17} & ${\approx}3{\times}10^{-4}\%$ (floor) \\
Unnormalized sum (AlphaEdit clients, $\gamma=1$) & $51.17 \pm 0.27$ & $51.07 \pm 0.29$ & ${\approx}3{\times}10^{-4}\%$ (floor) \\
\midrule
FedAlphaEdit (ours)        & \ensuremath{94.67 \pm 0.18} & \ensuremath{66.38 \pm 0.19} & \ensuremath{3.21\times10^{-4}\%} (floor) \\
FedAlphaEdit (five rounds) & \ensuremath{95.37 \pm 0.08} & \ensuremath{65.06 \pm 0.04} & ${\approx}3{\times}10^{-4}\%$ (floor) \\
\bottomrule
\end{tabular}
\caption{FedAlphaEdit against the centralized reference and the naive combination on GPT-2 XL (MCF). All rows edit the same 5{,}000 facts, single-round except the five-round variant ($10 \times 100 \times 5$) and the FedLEKE reimplementation (ten time slots). Centralized rows are single deterministic runs; all client-partitioned rows are three-seed mean$\pm$SD. NS is bounded above by the unedited model, which by construction attains NS $=100$, so a higher NS at a collapsed ES indicates weaker editing rather than better preservation, as the simple-average row illustrates. The simple-average and unnormalized-sum rows are the scalar merges $\gamma\sum_i \Delta_i$ with $\gamma = 1/N$ and $\gamma = 1$.}
\label{tab:g2res}
\end{table*}

Theorem~\ref{thm:exact} predicts that sharing projected statistics restores global AlphaEdit editing exactly in the one-shot regime; this section measures how much of that exactness survives genuine multi-layer sequential editing. Table~\ref{tab:g2res} contrasts centralized AlphaEdit, the naive combination, and FedAlphaEdit (three seeds). FedAlphaEdit restores neighborhood specificity to the centralized level and \oons\ to the measurement floor, and closes most of the edit-success gap. It returns \oons\ to the floor ($\ensuremath{3.21\times10^{-4}\%}$) and lifts neighborhood specificity to $\ensuremath{66.38 \pm 0.19}$, which is $1.95$ points \emph{above} the centralized reference. The floor-level \oons\ is enforced by construction, since every merged term carries a trailing $\Pnull$, so the measurement verifies that the implementation realizes the identity under genuine sequential editing, while the behavioral evidence for preservation rests on neighborhood specificity, GLUE, and zsRE. Editing is largely but not fully recovered: the edit success of FedAlphaEdit, $\ensuremath{94.67 \pm 0.18}$, stands $2.41$ points below the centralized reference, against the naive gap of $17.39$ points. By contrast with the naive collapse, the proposed series declines by only $1.28$ points in edit success from $N=2$ to $N=10$ (the same $N$-axis confound of Section~\ref{sec:exp-naive} applies), against $14.11$ for the naive combination over that range in Figure~\ref{fig:scaling}, while its \oons\ stays at the measurement floor for every $N$.

The $2.41$-point residual is analyzed in Section~\ref{sec:exp-diag}.

\noindent\textbf{Generality (Qwen).} Both the degradation and its recovery reproduce on a second architecture family, Qwen2.5-1.5B (Table~\ref{tab:qwen}). At a scale matched to GPT-2 XL ($10 \times 500$), the naive combination degrades on both axes, with an edit-success gap of $4.75$ points and \oons\ at $14.5\%$, and the proposed method restores both axes toward the centralized level. The residual in edit success left by the proposed method is $1.88$ points, the same direction as the $2.41$-point residual on GPT-2 XL, consistent with a scale-driven origin.

\begin{table}[t]
\centering
\begin{tabular}{lcc}
\toprule
Qwen2.5-1.5B & ES & NS \\
\midrule
Centralized  & $98.28$ & $73.32$ \\
Naive        & $93.53 \pm 0.58$ & $71.84 \pm 0.75$ \\
FedAlphaEdit & $96.40 \pm 0.27$ & $74.32 \pm 0.16$ \\
\bottomrule
\end{tabular}
\caption{Generality on Qwen2.5-1.5B (MCF, 10 clients $\times$ 500 edits). Centralized is a single deterministic run; Naive and FedAlphaEdit are three-seed mean$\pm$SD.}
\label{tab:qwen}
\end{table}

\begin{table*}[t]
\centering
\begin{tabular}{lccc}
\toprule
zsRE & Centralized & Naive & FedAlphaEdit \\
\midrule
\multicolumn{4}{l}{\emph{GPT-2 XL}} \\
Rewrite acc      & $89.40$ & $66.43 \pm 1.27$ & $89.40 \pm 0.41$ \\
Paraphrase acc   & $81.20$ & $62.22 \pm 0.39$ & $81.60 \pm 0.16$ \\
Neighborhood acc & $24.98$ & $23.72 \pm 0.26$ & $25.12 \pm 0.11$ \\
\oons            & $3.17{\times}10^{-4}\%$ (floor) & $15.75 \pm 0.04\%$ & $3.17{\times}10^{-4}\%$ (floor) \\
\midrule
\multicolumn{4}{l}{\emph{Qwen2.5-1.5B}} \\
Rewrite acc      & $90.77$ & $77.54 \pm 0.88$ & $90.37 \pm 0.17$ \\
Paraphrase acc   & $85.14$ & $75.19 \pm 1.33$ & $85.88 \pm 0.07$ \\
Neighborhood acc & $41.26$ & $41.00 \pm 0.57$ & $42.31 \pm 0.58$ \\
\oons            & $3.83{\times}10^{-4}\%$ (floor) & $17.10 \pm 0.13\%$ & $3.83{\times}10^{-4}\%$ (floor) \\
\bottomrule
\end{tabular}
\caption{Dataset generality on zsRE (each 10 clients $\times$ 200 edits, single round). Naive and FedAlphaEdit are three-seed mean$\pm$SD; Centralized is a single run; \oons\ of order $10^{-4}\%$ is the measurement floor. zsRE accuracies are first-token exact-match ($\times 100$); \oons\ is a weight-space five-layer mean. No numerical comparison with MCF or between the two model families is intended.}
\label{tab:zsre}
\end{table*}

\subsection{Ablation: Both Constraints Are Necessary}
\label{sec:exp-ablation}
Recovery requires both the local null-space constraint and the aligned merge. Table~\ref{tab:g2res} reports the four ablation arms over three seeds. The merge-only and the two scalar-merge arms keep \oons\ at the floor by construction of their right-projected or scalar merges, so the discriminating axes are ES and NS. \emph{Merge-only, the ``CollabEdit (MEMIT) + null-space projection'' row of Table~\ref{tab:g2res}.} In this arm each client solves with MEMIT without a local projection, and the server right-multiplies the covariance-weighted merge by $\Pnull$. Contrary to our expectation that this arm would preserve less than FedAlphaEdit, the outcome is an edit--specificity trade-off, with the merge-only arm above FedAlphaEdit on neighborhood specificity but $12.22$ points below on edit success, so a merge-side constraint alone suffices for specificity but not for editing, and the necessity of both constraints shows in the simultaneous attainment of the two, which only FedAlphaEdit achieves. \emph{Scalar merges (Proposition~\ref{prop:degenerate}).} Every scalar merge $\gamma\sum_i \Delta_i$ keeps \oons\ at the floor by construction, but no $\gamma$ recovers editing: the $1/N$ average collapses edit success to $31.23$, the unnormalized sum ($\gamma=1$) to $51.17$ with neighborhood specificity falling to $51.07$, and the intermediate $\gamma=0.5$ reaches only $55.16$ (Table~\ref{tab:scalar} in Appendix~\ref{app:exp}). Edit success stays at least $39$ points below FedAlphaEdit at every $\gamma$, so the recovery rests on the matrix-valued weights of Theorem~\ref{thm:exact} rather than on the scalar scale of the merged update. \emph{Five rounds.} As predicted, FedAlphaEdit holds its single-round level across rounds, in contrast with the naive round-axis amplification. In the multi-round regime each client accumulates only its own prior edits in $K_{p,i}$, so the bookkeeping premise of Theorem~\ref{thm:exact}, client prior second moments summing to the global one, no longer holds. The five-round arm is therefore a robustness check outside the guarantee rather than an instance of it.

\subsection{Generality beyond MCF: zsRE}
\label{sec:exp-zsre}
To confirm the phenomenon and its repair are not specific to MCF, we repeated the contrast on zsRE under direction predictions committed before the runs. On both GPT-2 XL and Qwen2.5-1.5B the naive combination reproduced the structural leak and collapse while null-space-aligned merging restored every metric to, or within 0.4 points of, the centralized level (Table~\ref{tab:zsre}). All zsRE evaluations use a common fp32 evaluation path within each model family (Appendix~\ref{app:exp}). Differences below the measurement resolution are reported as unresolvable, with the near-tie calibration behind this bound defined in Appendix~\ref{app:exp}. On Qwen$\times$zsRE the naive combination's behavioral neighborhood degradation is not resolvable ($-0.26$ points, below both the seed SD and the near-tie sensitivity band), although the structural leak itself reproduces clearly.

\subsection{Diagnostics}
\label{sec:exp-diag}

\noindent\textbf{Residual to the centralized reference.} The $2.41$-point residual is consistent across all three seeds ($94.88$, $94.54$, $94.60$; SD $0.18$). We attribute this residual to the layer-wise state divergence that the one-shot idealization excludes, with supporting diagnostics in Appendix~\ref{app:exp}. A single-layer control isolates the cause. Editing layer 17 alone, so that the common-state premise is met up to the batch dependence of the bf16 forward pass, the relative weight-space deviation falls from $0.4547$ to $0.0176$, and the behavioral metrics coincide within the run-to-run variation of the centralized reference (edit success within $0.10$ points, NS identical). The remaining deviation of $0.0176$ is not located in the merge identity but in the operator inputs: the relative difference between the merged normalizer $N_{\text{all}}$ and the normalizer $N_G$ of the pooled centralized solve, $\Fnorm{N_{\text{all}} - N_G}/\Fnorm{N_G}$, is $0.0063$ in this control, and since $N_{\text{all}} - N_G = \Pnull(\sum_i A_i - A_G)\Pnull$ with the ridge cancelling, this difference lives entirely in the key statistics. Both arms of the control use the same edit requests, the same fixed context templates, and the same cached target hidden states, so the only difference between them is the batch composition under which the keys are computed, and the batch dependence of the key statistics in the bf16 forward pass is where the residual resides.

A natural concern is that exceeding the centralized reference on neighborhood specificity while falling short on edit success is the signature of an edit that writes less. Four observations argue otherwise. First, FedAlphaEdit keeps \oons\ at the measurement floor, so preservation is upheld by the null-space constraint itself, and the $1.95$-point NS gain over centralized AlphaEdit is not an artifact of unmeasured leakage. Second, genuine under-editing is visible in the simple-average ablation (Section~\ref{sec:exp-ablation}, Table~\ref{tab:g2res}), whose edit success collapses at an elevated neighborhood specificity. By contrast, FedAlphaEdit operates in a qualitatively different regime in which the editing function is substantially retained. Third, the residual is not a deliberate reduction of edit magnitude. Its error decomposes (Appendix~\ref{app:exp}) entirely inside $\operatorname{range}(\Pnull)$, the subspace the editor writes into, with the orthogonal component at the \oons\ floor, so the $2.41$-point shortfall is a directional distortion introduced by layer-wise state divergence rather than a uniform shrinkage of the edits. Fourth, as a supportive, non-binding observation, general ability points the same way. Under globally weakened editing GLUE would drift toward the base level, as it does for the naive combination (Section~\ref{sec:exp-naive}), whereas FedAlphaEdit reproduces the centralized level ($0.3732$ versus $0.3728$, Table~\ref{tab:glue} in Appendix~\ref{app:exp}). These observations are inconsistent with uniform shrinkage and with unmeasured leakage. Excluding a per-edit trade-off within the projected subspace entirely would require edit-level analysis, which we leave open.

\noindent\textbf{Prior-method reference (MEMIT).} The MEMIT control and Centralized MEMIT rows of Table~\ref{tab:g2res} agree to within seed-level variation, consistent with the cancellation identity of Lemma~\ref{lem:cancel}(a). The FedLEKE row of Table~\ref{tab:g2res} is our reimplementation from the paper description (gate threshold estimated by us, and not evaluated on GPT-2 XL or MCF by the original) and is a reference under our internal-contrast pipeline rather than a claim about the original system; its \oons\ of $23.98\%$ reflects the absence of a null-space constraint rather than the breaking of one.

\section{Limitations}
\label{sec:limitations}
\begin{itemize}
\item \textbf{Idealization of Theorem~\ref{thm:exact}.} Exactness assumes a one-shot solve and all key statistics at a common model state; under multi-layer sequential editing the identity holds only approximately, and we disclose the measured drift (Appendix~\ref{app:exp}). Even in the single-layer control, where the common-state premise is met up to the batch dependence of the bf16 forward pass, a weight-space deviation of $0.0176$ remains and is located in the operator inputs rather than in the merge identity (Section~\ref{sec:exp-diag}).
\item \textbf{Edit-success residual on GPT-2 XL.} The proposed method leaves a $2.41$-point residual to the centralized reference in edit success. We attribute this to the layer-wise state divergence that Theorem~\ref{thm:exact} explicitly excludes, and a protocol that evaluates key statistics at a common model state is left to future work. At the matched scale ($10 \times 500$), both the naive amplification and a residual of the proposed method reproduce on Qwen, which partially disentangles the confound with per-client edit count, while confounds with model scale and edit-layer depth remain.
\item \textbf{Scale.} We study two 1.5B-parameter families under a single-GPU budget; larger models, non-IID sharding for the proposed method (Theorem~\ref{thm:exact} itself is partition-independent, Section~\ref{sec:theorem}), imbalanced or conflicting edits across clients, and client dropout are left to future work. The shared statistic is a dense $d_0 \times d_0$ matrix per edited layer, a substantial communication footprint (quantified in Appendix~\ref{app:exp}), and low-rank or sketched compression is untested. The magnitude of the naive collapse itself varies across architectures ($17.39$ points on GPT-2 XL versus $4.75$ on Qwen), and isolating the contributions of layer depth, key dimension, and $\lambda$ to this difference is left to future work.
\item \textbf{Shared projector.} Theorem~\ref{thm:exact} assumes a single projector $\Pnull$ shared by all clients, constructed here from common Wikipedia statistics. This is natural when clients hold similar preservation requirements, as with several hospitals maintaining one clinical model, but how clients with heterogeneous preservation requirements would agree on a shared projector, and whether a projector built from public statistics covers each client's private requirements, are left open.
\item \textbf{Privacy.} Sharing the projected covariance statistic $\Pnull A_i \Pnull$ avoids transmitting raw keys but is not by itself a privacy guarantee. When a client holds few edits the statistic is low-rank, and the projected key directions can be recovered from it up to sign, with $\Delta_i$ potentially revealing further information about the residuals. The protocol further assumes an honest-but-curious server and defines no formal threat model. The server needs only two sums, $\sum_i \Delta_i(\Pnull A_i \Pnull + \beta I)$, which equals $\sum_i R_i K_i^{\top}\Pnull$ by Lemma~\ref{lem:rightproj}, and $\sum_i \Pnull A_i \Pnull$, and each summand is computable locally by its client, so the dense-matrix form is in principle compatible with additive secure aggregation, although we have neither implemented nor evaluated it. The factorized transmission of $\Pnull K_i$ described in Appendix~\ref{app:exp} lacks this property, because the server must form $\Pnull K_i K_i^{\top}\Pnull$ from each client's individual factor. Reconstruction-attack evaluation, secure aggregation, and differentially private transmission are left to future work.
\end{itemize}

\section{Conclusion}
\label{sec:conclusion}
We showed that naively merging null-space-constrained edits from independent clients fails structurally, with editing and preservation collapsing at once because the covariance-weighted merge is incompatible with AlphaEdit's projected normalizer. Guided by this analysis, FedAlphaEdit has clients share projected statistics and provably recovers global AlphaEdit editing in the one-shot regime; under genuine sequential editing the recovery is approximate, with the drift disclosed. Empirically it restores neighborhood specificity to the centralized level and \oons\ to the measurement floor, and closes most of the edit-success gap on both architecture families, with the small remaining residual attributed to the layer-wise state divergence that the idealization of Theorem~\ref{thm:exact} excludes.

\bibliography{references}

\begin{thebibliography}{14}
\providecommand{\natexlab}[1]{#1}

\bibitem[{Eskandar et~al.(2026)Eskandar, Sirera~Perell\'{o}, Ioannidis, and
  Dy}]{eskandar2026loki}
Eskandar, M.; Sirera~Perell\'{o}, M.; Ioannidis, S.; and Dy, J. 2026.
\newblock {LOKI}: Memory-Free Null-Space Constrained Lifelong Knowledge
  Editing.
\newblock arXiv:2606.19679.

\bibitem[{Fang et~al.(2025)Fang, Jiang, Wang, Ma, Shi, Wang, He, and
  Chua}]{fang2025alphaedit}
Fang, J.; Jiang, H.; Wang, K.; Ma, Y.; Shi, J.; Wang, X.; He, X.; and Chua,
  T.-S. 2025.
\newblock {AlphaEdit}: Null-Space Constrained Knowledge Editing for Language
  Models.
\newblock In \emph{International Conference on Learning Representations
  (ICLR)}.

\bibitem[{Ilharco et~al.(2023)Ilharco, Ribeiro, Wortsman, Gururangan, Schmidt,
  Hajishirzi, and Farhadi}]{ilharco2023task}
Ilharco, G.; Ribeiro, M.~T.; Wortsman, M.; Gururangan, S.; Schmidt, L.;
  Hajishirzi, H.; and Farhadi, A. 2023.
\newblock Editing Models with Task Arithmetic.
\newblock In \emph{International Conference on Learning Representations
  (ICLR)}.

\bibitem[{Liu, Liu, and Li(2026)}]{liu2026betaedit}
Liu, B.; Liu, W.; and Li, Y. 2026.
\newblock {BetaEdit}: Null-Space Constrained Sequential Model Editing.
\newblock In \emph{Proceedings of the International Joint Conference on
  Artificial Intelligence (IJCAI)}.

\bibitem[{Lyu et~al.(2026)Lyu, Gu, Wang, Huang, Luan, Cui, Chang, and
  Lu}]{lyu2026evoedit}
Lyu, S.; Gu, Y.; Wang, X.; Huang, J.; Luan, S.; Cui, Y.; Chang, X.-W.; and Lu,
  P. 2026.
\newblock {EvoEdit}: Evolving Null-space Alignment for Robust and Efficient
  Knowledge Editing.
\newblock In \emph{Findings of the Association for Computational Linguistics:
  ACL 2026}.

\bibitem[{McMahan et~al.(2017)McMahan, Moore, Ramage, Hampson, and {Ag{\"u}era
  y Arcas}}]{mcmahan2017fedavg}
McMahan, B.; Moore, E.; Ramage, D.; Hampson, S.; and {Ag{\"u}era y Arcas}, B.
  2017.
\newblock Communication-Efficient Learning of Deep Networks from Decentralized
  Data.
\newblock In \emph{Proceedings of the International Conference on Artificial
  Intelligence and Statistics (AISTATS)}, 1273--1282.

\bibitem[{Meng et~al.(2022)Meng, Bau, Andonian, and Belinkov}]{meng2022rome}
Meng, K.; Bau, D.; Andonian, A.; and Belinkov, Y. 2022.
\newblock Locating and Editing Factual Associations in {GPT}.
\newblock In \emph{Advances in Neural Information Processing Systems
  (NeurIPS)}, volume~35, 17359--17372.

\bibitem[{Meng et~al.(2023)Meng, Sharma, Andonian, Belinkov, and
  Bau}]{meng2023memit}
Meng, K.; Sharma, A.~S.; Andonian, A.; Belinkov, Y.; and Bau, D. 2023.
\newblock Mass-Editing Memory in a Transformer.
\newblock In \emph{International Conference on Learning Representations
  (ICLR)}.

\bibitem[{Qiu et~al.(2026)Qiu, Wang, Cao, Zhang, Su, Xu, Meng, Xu, Wu, and
  Li}]{qiu2026nufilt}
Qiu, Z.; Wang, L.; Cao, Y.; Zhang, R.; Su, B.; Xu, Y.; Meng, F.; Xu, L.; Wu,
  Q.; and Li, H. 2026.
\newblock Null-Space Filtering for Data-Free Continual Model Merging:
  Preserving Stability, Promoting Plasticity.
\newblock In \emph{International Conference on Learning Representations
  (ICLR)}.

\bibitem[{Qiu et~al.(2025)Qiu, Xu, He, Meng, Xu, Wu, and Li}]{qiu2025mingle}
Qiu, Z.; Xu, Y.; He, C.; Meng, F.; Xu, L.; Wu, Q.; and Li, H. 2025.
\newblock {MINGLE}: Mixture of Null-Space Gated Low-Rank Experts for Test-Time
  Continual Model Merging.
\newblock In \emph{Advances in Neural Information Processing Systems
  (NeurIPS)}.

\bibitem[{Yadav et~al.(2023)Yadav, Tam, Choshen, Raffel, and
  Bansal}]{yadav2023ties}
Yadav, P.; Tam, D.; Choshen, L.; Raffel, C.; and Bansal, M. 2023.
\newblock {TIES-Merging}: Resolving Interference When Merging Models.
\newblock In \emph{Advances in Neural Information Processing Systems
  (NeurIPS)}, volume~36, 7093--7115.

\bibitem[{Yao et~al.(2023)Yao, Wang, Tian, Cheng, Li, Deng, Chen, and
  Zhang}]{yao2023editing}
Yao, Y.; Wang, P.; Tian, B.; Cheng, S.; Li, Z.; Deng, S.; Chen, H.; and Zhang,
  N. 2023.
\newblock Editing Large Language Models: Problems, Methods, and Opportunities.
\newblock In \emph{Proceedings of the 2023 Conference on Empirical Methods in
  Natural Language Processing}, 10222--10240.

\bibitem[{Zhao et~al.(2025)Zhao, Xu, Li, Wei, and Zhong}]{zhao2025fleke}
Zhao, Z.; Xu, G.; Li, X.; Wei, K.; and Zhong, J. 2025.
\newblock {FedLEKE}: Federated Locate-then-Edit Knowledge Editing.
\newblock In \emph{Findings of the Association for Computational Linguistics:
  ACL 2025}, 14247--14258.

\bibitem[{Zheng et~al.(2025)Zheng, Zhang, Du, Zhang, Yin, and
  Lin}]{zheng2025collabedit}
Zheng, J.; Zhang, J.; Du, T.; Zhang, X.; Yin, J.; and Lin, T. 2025.
\newblock {CollabEdit}: Towards Non-destructive Collaborative Knowledge
  Editing.
\newblock In \emph{International Conference on Learning Representations
  (ICLR)}.

\end{thebibliography}

\clearpage
\renewcommand{\thesection}{\Alph{section}}
\setcounter{section}{0}
\setcounter{secnumdepth}{1}

\section{Proofs}
\label{app:proof-thm1}
\begin{proof}[Proof of Lemma~\ref{lem:rightproj}]
Since $\Pnull^2 = \Pnull$ and $\Pnull^{\top} = \Pnull$, we have $\Pnull(A\Pnull + \beta I) = \Pnull A \Pnull + \beta \Pnull = (\Pnull A \Pnull + \beta I)\Pnull$. Both $A\Pnull + \beta I$ and $\Pnull A \Pnull + \beta I$ are invertible (the nonzero spectra of $A\Pnull$ and $\Pnull A \Pnull$ coincide, so both shifts by $\beta I$ are invertible for $\beta > 0$); multiplying on the right by $(A\Pnull+\beta I)^{-1}$ and on the left by $(\Pnull A \Pnull + \beta I)^{-1}$ gives $\Pnull(A\Pnull+\beta I)^{-1} = (\Pnull A \Pnull + \beta I)^{-1}\Pnull$. Hence $\Delta = R K^{\top}\Pnull(A\Pnull+\beta I)^{-1} = [R K^{\top}(\Pnull A \Pnull + \beta I)^{-1}]\,\Pnull =: D\Pnull$.
\end{proof}

\begin{proof}[Proof of Theorem~\ref{thm:exact}]
By Lemma~\ref{lem:rightproj}, $\Delta_i(\Pnull A_i \Pnull + \beta I) = R_i K_i^{\top}\Pnull$ (commuting $\Pnull$ through the inverse, the inverse and the normalizer cancel). Summing, $\sum_i \Delta_i(\Pnull A_i \Pnull + \beta I) = \sum_i R_i K_i^{\top}\Pnull$. Since $N_{\text{all}}$ has the form $\Pnull X \Pnull + \beta I$, it commutes with $\Pnull$ as in Lemma~\ref{lem:rightproj}, so $N_{\text{all}}^{-1}\Pnull = \Pnull N_{\text{all}}^{-1}$, and
\begin{align*}
\Delta_{\text{proposed}} &= \Big(\sum_i R_i K_i^{\top}\Big)\Pnull\, N_{\text{all}}^{-1} \\
&= \Big(\sum_i R_i K_i^{\top}\Big) N_{\text{all}}^{-1}\Pnull = \Delta_{\text{central}}^{\mathrm{Alpha}},
\end{align*}
the last equality being global AlphaEdit editing with pooled residual $\sum_i R_i K_i^{\top}$ and normalizer $N_{\text{all}}$. Right-multiplication by $(I-\Pnull)$ annihilates the trailing $\Pnull$, giving exact preservation. Since the identity $\Delta_{\text{proposed}} = \Delta_{\text{central}}^{\mathrm{Alpha}}$ holds irrespective of the value of $\Pnull K_0$, the merged leakage $\Delta_{\text{proposed}} K_0$ coincides with the centralized leakage for every $\varepsilon$, containing no structural term.
\end{proof}

\noindent\textbf{Remark (general regularizer, Proposition~\ref{prop:editormatched}).} If $\beta I$ is replaced by a symmetric positive definite $\Lambda$ with $\Pnull\Lambda = \Lambda\Pnull$, the identity $\Pnull(A\Pnull+\Lambda) = (\Pnull A\Pnull+\Lambda)\Pnull$ in the proof of Lemma~\ref{lem:rightproj} follows from this commutation alone. Invertibility follows from the decomposition $\mathbb{R}^{d_0} = \operatorname{range}(\Pnull) \oplus \operatorname{range}(I-\Pnull)$. Under it $\Lambda$ is block diagonal and $A\Pnull+\Lambda$ is block lower triangular, with diagonal blocks $\Pnull A\Pnull + \Pnull\Lambda\Pnull$ and $(I-\Pnull)\Lambda(I-\Pnull)$, restricted to the respective subspaces and both positive definite. The same commutation gives $N_{\text{all}}^{-1}\Pnull = \Pnull N_{\text{all}}^{-1}$ in the proof of Theorem~\ref{thm:exact}, so both proofs carry over verbatim.

\section{Released Implementation Form and Notation Reconciliation}
\label{app:impl}
For the theory-to-implementation cross-check, the per-layer expression synthesized by the released code (CPU, \texttt{float64}) is
\begin{multline}
\Delta_{\text{naive}} = \Big[\textstyle\sum_i \mathrm{resid}_i\,(S_i^{-1}\Pnull K_i)^{\top}(K_i K_i^{\top} + \lambda C)\Big]\\
\cdot(\lambda C + \textstyle\sum_j K_j K_j^{\top})^{-1},
\label{eq:released}
\end{multline}
where $S_i := \Pnull(K_i K_i^{\top} + C_p^{(i)}) + \beta I$ is the released \emph{solve operator} ($\beta = \text{L2} = 10$ the AlphaEdit ridge; $C_p^{(i)} = 0$ in the single-round setting), $\lambda$ the second-moment update weight ($\lambda = 20000$ for GPT-2 XL and $\lambda = 15000$ for Qwen2.5-1.5B, distinct from $\beta$), and $C$ the shared Wikipedia second-moment $\mathbb{E}[kk^{\top}]$ (sample-normalized). The correspondence to equation~\eqref{eq:merge-abstract} is $\Delta_i = \mathrm{resid}_i(S_i^{-1}\Pnull K_i)^{\top}$ and $C_i = K_i K_i^{\top}$; the released weighting uses the regularized numerator $C_i + \lambda C$ and denominator $\lambda C + \sum_j K_j K_j^{\top} = C_{\text{tot}} + \lambda C$, so $\sum_i (C_i + \lambda C) = C_{\text{tot}} + N\lambda C$ over-counts $\lambda C$ by $(N-1)\lambda C$ relative to the denominator (Section~\ref{sec:lemma4}).

\paragraph{Notation clash (important).} The released code's solve operator $S_i = \Pnull(K_iK_i^{\top}+C_p^{(i)}) + \beta I$ is \emph{one-sided}; it corresponds to the paper's normalizer $N_i = \Pnull A_i \Pnull + \beta I$ of Section~\ref{sec:prelim}, \emph{not} to the paper's symmetric second-moment matrix $A_i = K_{p,i}K_{p,i}^{\top}+K_iK_i^{\top}$. In particular, the variable named ``$A_i$'' in the released code is $S_i$ here, a different object from the paper's $A_i$.

\paragraph{Equivalence (one line).} With $A_i = K_{p,i}K_{p,i}^{\top}+K_iK_i^{\top}$ we have $S_i^{\top} = A_i\Pnull + \beta I$, so $\Delta_i = \mathrm{resid}_i(S_i^{-1}\Pnull K_i)^{\top} = \mathrm{resid}_i K_i^{\top}\Pnull(A_i\Pnull+\beta I)^{-1}$; by Lemma~\ref{lem:rightproj}, $\Pnull(A_i\Pnull+\beta I)^{-1} = (\Pnull A_i\Pnull+\beta I)^{-1}\Pnull = N_i^{-1}\Pnull$, hence $\Delta_i = \mathrm{resid}_i K_i^{\top}N_i^{-1}\Pnull \in \operatorname{range}(\Pnull)$. The one-sided operator $S_i$ and the sandwiched normalizer $N_i$ therefore yield the identical update \emph{precisely because} the update is right-projected onto $\operatorname{range}(\Pnull)$ (the trailing $\Pnull$).

\section{Experimental Details}
\label{app:exp}

\begin{table*}[t]
\centering
\begin{tabular}{lcccc}
\toprule
Per-client $500$ & ES & Centralized ES & NS & \oons \\
\midrule
$N=2$ ($1000$)  & $99.83 \pm 0.12$ & $99.80$ & $68.37 \pm 0.07$ & $4.05\%$ \\
$N=5$ ($2500$)  & $95.03 \pm 0.15$ & $99.24$ & $62.80 \pm 0.17$ & $9.52\%$ \\
$N=10$ ($5000$) & $79.69 \pm 0.88$ & $97.08$ & $60.24 \pm 0.11$ & $14.04\%$ \\
\bottomrule
\end{tabular}
\caption{Per-client-fixed control series (naive combination, GPT-2 XL, per-client $500$); three-seed mean$\pm$SD, with \oons\ standard deviations $0.027$, $0.057$, and $0.022$ percentage points. The $N{=}10$ row reuses the main naive runs. Centralized entries are deterministic single runs editing the same case sets as the corresponding rows; the $N{=}10$ centralized entry reuses the main centralized run. Total edits co-vary with $N$; cross-column comparisons are within the same row (matched total edit count) only.}
\label{tab:npc}
\end{table*}

\begin{itemize}
\item \textbf{Reported records.} Per-seed edit-success values for FedAlphaEdit appear in Section~\ref{sec:exp-diag}, three-seed conditions are reported as mean\,$\pm$\,SD throughout, and the GLUE summary is Table~\ref{tab:glue} (the six tasks, ported from the AlphaEdit evaluation harness, are SST-2, MRPC, CoLA, RTE, MMLU, and a binary NLI set, $n=100$ per task, zero-shot). Full per-seed records for every run are retained with the experimental logs.
\item \textbf{Shared covariance.} We keep $C_1 = 0$, the released default, and do not implement Dynamic Covariance, because no working public implementation exists and a bespoke one would introduce an unverified baseline rather than a faithful reproduction.
\item \textbf{Measurement-uncertainty calibration.} The run-to-run edit-success variation of the centralized reference is $0.12$ points (NS: $0.02$ points), and the cross-process floor of the operator-divergence metric (from seed-free template generation) is $9$--$12\%$; these calibrate the drift diagnostics below and the residual discussion of Section~\ref{sec:exp-diag} and are not used in any claim.
\item \textbf{Qwen environment.} The mom2 statistics source (Wikipedia 20220301.en) and other Qwen-specific settings are held common across all Qwen runs so that internal contrasts remain valid. Editing uses a $v$-loss layer of $27$ and $v$-learning-rate $0.1$ over $25$ steps.
\item \textbf{Evaluation pipeline.} The pipeline exhibits a systematic $+1.8$ to $+2.0$ point paraphrase shift, common to all methods and excluded as a bf16 artifact; it motivates the internal-contrast principle and is never compared across pipelines.
\item \textbf{Measurement-system sanity.} Under an unprojected MEMIT editor, \oons\ is $23$--$31\%$ (a 20-edit smoke sample, a measurement-system sanity check, not a scaled comparison), about $10^{5}\times$ the projected editor's floor, confirming that the metric resolves the out-of-null-space norm share.
\item \textbf{Unit test of Theorem~\ref{thm:exact}.} One-shot single-client exactness ($\ensuremath{3.4\times10^{-15}}$) and the two-client divergence ($\ensuremath{2.4\%}$, from a $\ensuremath{0.6\%}$ operator difference) provide direct evidence of the layer-direction term discussed under Drift diagnostics below. As a negative control, adding the ridge $\beta I$ once per client instead of once to $N_{\text{all}}$ (the $(N{-}1)$-fold over-counting that Theorem~\ref{thm:exact} forbids) diverges by $\ensuremath{11.8\%}$, far above the correctly constructed two-client value ($\ensuremath{2.4\%}$). The test therefore detects a $\beta$-mismatch, a violation of the theorem's precondition, rather than passing trivially.
\item \textbf{Communication footprint.} Per edited layer the shared statistic $\Pnull A_i \Pnull$ is a dense $d_0 \times d_0$ matrix, which in fp32 ($4$ bytes per element) is $163.8$\,MB for GPT-2 XL ($d_0{=}6400$) and $321.1$\,MB for Qwen2.5-1.5B ($d_0{=}8960$), i.e.\ $819.2$\,MB and about $1.6$\,GB per client per round over the five edited layers. This matches the dimension of CollabEdit's key-covariance statistic. In the single-round setting the statistic factorizes: $\operatorname{rank}(\Pnull A_i \Pnull) \le n_i$, so transmitting the projected keys $\Pnull K_i$ ($d_0 \times n_i$) is mathematically equivalent and, for $n_i = 500$, amounts to $d_0 \times n_i$ fp32 values, or $12.8$\,MB per layer for GPT-2 XL (an arithmetic value, not a measurement). This variant, however, transmits projected keys directly and therefore trades against the reconstruction risk noted under Privacy in Section~\ref{sec:limitations}; we record it not as a communication-reduction claim but as the fact that the communication and privacy limitations share the same degree of freedom.
\item \textbf{Software environment.} The software stack is Driver 591.86, CUDA (driver) 13.1, and PyTorch 2.6.0+cu118 (runtime CUDA 11.8).
\item \textbf{Seeds and sharding.} All three-seed conditions use shard seeds $1$, $2$, $3$, which control IID assignment of cases to clients under a fixed case universe (MCF first $5000$ or $2000$ records). Centralized runs are deterministic (no sharding).
\item \textbf{Hyperparameter provenance.} The editing hyperparameters ($\lambda$, $\beta$, threshold, edited layers, $v$-learning-rate, and optimization steps) are inherited from the released AlphaEdit and CollabEdit configurations without independent tuning. The sensitivity of the shared-covariance weight $\lambda$ is examined at three levels ($\{10000, 20000, 40000\}$) in Table~\ref{tab:lamc}.
\end{itemize}

\noindent\textbf{Non-IID sharding (naive combination).} Under relation-based non-IID sharding of the same $5000$ edits ($10 \times 500$, seed 1), the naive combination reaches ES $80.06$, NS $61.94$, and \oons\ $12.0\%$, so the naive collapse persists at a magnitude comparable to IID sharding (Table~\ref{tab:g1}); the proposed method has not been measured under non-IID sharding.

\paragraph{Drift diagnostics.}
\label{app:exp-diag}
Theorem~\ref{thm:exact} is exact only when all key statistics are evaluated at a common model state. Under multi-layer sequential editing the measured \emph{sequential drift} is $\ensuremath{0.4547}$ (three seeds, five-layer mean), localized by a \emph{layer-wise operator divergence}, both increasing monotonically across the edited layers and peaking at the last (Figure~\ref{fig:layerdiag}). A single-client, one-shot unit test reproduces the exactness identity to machine precision, while its two-client version already exhibits a divergence traceable to an operator difference, evidence in the direction of the state-divergence attribution. The single-layer control of Section~\ref{sec:exp-diag} bounds how much of this drift is behaviorally consequential. With sequential multi-layer evaluation removed, the deviation drops to $0.0176$ (a factor of $26$) and the behavioral gap closes, consistent with the layer-wise attribution.

\paragraph{Why the residual costs editing but not preservation.} Decomposing the error $E=\Delta_{\text{proposed}}-\Delta_{\text{central}}^{\mathrm{Alpha}}$ as in Lemma~\ref{lem:cancel}(b), the $E\Pnull$ component accounts for $100.0\%$ of the error across all layers and seeds, while the orthogonal $E(I-\Pnull)$ component stays at the \oons\ floor ($3.2\times10^{-4}\%$). The residual thus perturbs edit strength inside the subspace the editor writes into, leaving preserved knowledge untouched.

\begin{figure}[t]
\centering
\includegraphics[width=0.82\columnwidth]{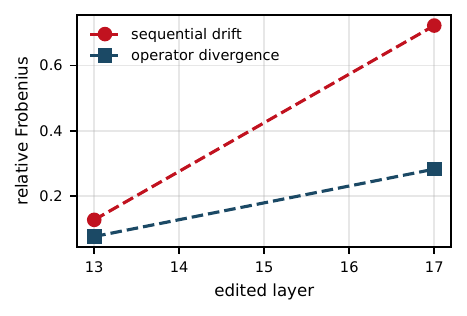}
\caption{Sequential drift $\Fnorm{\Delta_{\text{proposed}} - \Delta_{\text{central}}^{\mathrm{Alpha}}}/\Fnorm{\Delta_{\text{central}}^{\mathrm{Alpha}}}$ and operator divergence $\Fnorm{N_{\text{all}} - N_G}/\Fnorm{N_G}$ increase monotonically across the edited layers ($13$--$17$), peaking at the last; five-layer means $0.4547$ and $0.1879$. Markers show the layer-$13$ and layer-$17$ endpoints.}
\label{fig:layerdiag}
\end{figure}

\begin{table}[t]
\centering
\begin{tabular}{lc}
\toprule
 & GLUE avg.\ F1 \\
\midrule
base         & $0.4013$ \\
Centralized  & $0.3728$ \\
Naive        & $0.4087$ \\
FedAlphaEdit & $0.3732 \pm 0.0161$ \\
\bottomrule
\end{tabular}
\caption{GLUE average F1 over six tasks ($n=100$ per task), a supportive, non-binding observation. FedAlphaEdit is the three-seed mean$\pm$SD; base and Centralized are single values, and Naive is a three-seed mean. The naive combination sits nearer the base level than centralized editing does, consistent with the distorted merge writing edits more weakly (Section~\ref{sec:exp-naive}) rather than an advantage.}
\label{tab:glue}
\end{table}

\paragraph{Shared-covariance weight sensitivity.}
Varying the shared-covariance weight $\lambda$ (the second-moment update weight applied to the base covariance $C$) across $\{10000, 20000, 40000\}$ on the naive combination (GPT-2~XL, $10 \times 500$, seed~1) yields Table~\ref{tab:lamc}.
\begin{table}[t]
\centering
\begin{tabular}{lccc}
\toprule
$\lambda$ & ES & NS & \oons \\
\midrule
$10000$ & $84.28$ & $62.19$ & $15.80\%$ \\
$20000$ & $78.84$ & $60.12$ & $14.05\%$ \\
$40000$ & $73.72$ & $58.05$ & $11.98\%$ \\
\bottomrule
\end{tabular}
\caption{Sensitivity of the naive combination to the shared-covariance weight $\lambda$ (GPT-2 XL, MCF, $10 \times 500$, seed~1). Varying $\lambda$ trades edit success and specificity against \oons; no setting recovers both axes toward the centralized level.}
\label{tab:lamc}
\end{table}

\begin{table}[t]
\centering
\small
\setlength{\tabcolsep}{4pt}
\begin{tabular}{lccc}
\toprule
$\gamma$ & ES & NS & \oons \\
\midrule
$0.1$ ($=1/N$) & $31.23 \pm 0.22$ & $77.16 \pm 0.17$ & $3.143{\times}10^{-4}\%$ \\
$0.5$          & $55.16 \pm 0.47$ & $54.98 \pm 0.42$ & $3.143{\times}10^{-4}\%$ \\
$1.0$          & $51.17 \pm 0.27$ & $51.07 \pm 0.29$ & $3.143{\times}10^{-4}\%$ \\
\bottomrule
\end{tabular}
\caption{Scalar merges $\gamma\sum_i \Delta_i$ (GPT-2 XL, MCF, $10 \times 500$), three-seed mean$\pm$SD. The per-client AlphaEdit updates of the simple-average arm are reused and only evaluated, with no new solves. The row $\gamma = 0.1$ is the simple-average row of Table~\ref{tab:g2res}. Only three levels of $\gamma$ were tested, and finer tuning is untested. \oons\ is at the measurement floor in every row by construction of a scalar merge.}
\label{tab:scalar}
\end{table}

\paragraph{Near-tie calibration.}
First-token decisions on the Qwen zsRE prompts are compared between a single fp32 evaluation pass and chunked passes over $694$ per-prompt Boolean outcomes. Of these, $3/694$ invert (one at chunk size $2$, two at chunk size $8$). The minimum fp32 top-1/top-2 logit gap is $0.000467$, below the bf16 unit-in-last-place granularity at head scale, which is the mechanism of the inversions and sets the sensitivity band used in Section~\ref{sec:exp-zsre}.

\end{document}